\documentclass[11pt]{article}
\usepackage{fullpage}
\usepackage{amsmath,amsthm,amsfonts,amssymb,dsfont}

\usepackage[round]{natbib}

\usepackage{hyperref,cleveref}

\Crefformat{equation}{#2(#1)#3}
\crefformat{equation}{#2(#1)#3}
\Crefrangeformat{equation}{#3(#1)#4--#5(#2)#6}
\crefrangeformat{equation}{#3(#1)#4--#5(#2)#6}
\Crefmultiformat{equation}{#2(#1)#3}{ and~#2(#1)#3}{, #2(#1)#3}{ and~#2(#1)#3}
\crefmultiformat{equation}{#2(#1)#3}{ and~#2(#1)#3}{, #2(#1)#3}{ and~#2(#1)#3}

\usepackage{enumitem}
\setlist[description]{font=\normalfont}

\usepackage{csquotes}
\MakeOuterQuote{"}

\usepackage{mathtools}
\mathtoolsset{centercolon}
\DeclarePairedDelimiterXPP\ind[1]{\mathds{1}}{\lbrace}{\rbrace}{}{#1}
\DeclarePairedDelimiterX\eval[1]{\lbrace}{\rvert}{#1 \delimsize\rbrace}
\DeclarePairedDelimiter\ip{\langle}{\rangle}
\DeclarePairedDelimiter\abs{\lvert}{\rvert}
\DeclarePairedDelimiter\card{\lvert}{\rvert}

\DeclarePairedDelimiter\del{\lparen}{\rparen}

\DeclarePairedDelimiter\cbr{\lbrace}{\rbrace}
\DeclarePairedDelimiter\set{\lbrace}{\rbrace}

\DeclarePairedDelimiter\intcc{\lbrack}{\rbrack}

\newtheorem{theorem}{Theorem}

\DeclareMathOperator\softmax{softmax}

\newcommand\T{{\scriptscriptstyle{\mathsf{T}}}}
\newcommand\R{\mathbb{R}}
\newcommand\N{\mathbb{N}}

\newcommand\Sigmain{\Sigma_{\operatorname{in}}}
\newcommand\Sigmaout{\Sigma_{\operatorname{out}}}
\newcommand\encode{\psi_{\operatorname{in}}}
\newcommand\decode{\psi_{\operatorname{out}}}
\newcommand\TF{\mathsf{TF}}
\newcommand\SA{\mathsf{SA}}

\def\ddefloop#1{\ifx\ddefloop#1\else\ddef{#1}\expandafter\ddefloop\fi}
\def\ddef#1{\expandafter\def\csname bf#1\endcsname{\ensuremath{\mathbf{#1}}}}
\ddefloop abcdefghijklmnopqrstuvwxyzABCDEFGHIJKLMNOPQRSTUVWXYZ\ddefloop
\def\ddef#1{\expandafter\def\csname bf#1\endcsname{\ensuremath{\boldsymbol{\csname #1\endcsname}}}}
\ddefloop {alpha}{beta}{gamma}{delta}{epsilon}{varepsilon}{zeta}{eta}{theta}{vartheta}{iota}{kappa}{lambda}{mu}{nu}{xi}{pi}{varpi}{rho}{varrho}{sigma}{varsigma}{tau}{upsilon}{phi}{varphi}{chi}{psi}{omega}{Gamma}{Delta}{Theta}{Lambda}{Xi}{Pi}{Sigma}{varSigma}{Upsilon}{Phi}{Psi}{Omega}{ell}\ddefloop
\def\ddef#1{\expandafter\def\csname cal#1\endcsname{\ensuremath{\mathcal{#1}}}}
\ddefloop ABCDEFGHIJKLMNOPQRSTUVWXYZ\ddefloop

\title{One-layer transformers fail to solve the induction heads task}

\author{Clayton Sanford\thanks{Columbia University, \texttt{clayton@cs.columbia.edu}.} \and Daniel Hsu\thanks{Columbia University, \texttt{djhsu@cs.columbia.edu}.} \and Matus Telgarsky\thanks{New York University, \texttt{mjt10041@nyu.edu}.}}

\begin{document}
\maketitle

\begin{abstract}
  A simple communication complexity argument proves that no one-layer transformer can solve the induction heads task unless its size is exponentially larger than the size sufficient for a two-layer transformer.
\end{abstract}

\section{Introduction}

The mechanistic interpretability studies of \citet{elhage2021mathematical} and \citet{olsson2022context} identified the ubiquity and importance of so-called "induction heads" in transformer-based language models~\citep{vaswani2017attention,radford2019language,brown2020language}.
The basic task performed by an induction head is as follows.
\begin{itemize}
  \item The input is an $n$-tuple of tokens $(\sigma_1, \dotsc, \sigma_n)$ from a finite alphabet $\Sigma$.

  \item The output is an $n$-tuple of tokens $(\tau_1, \dotsc, \tau_n)$ from the augmented alphabet $\Sigma \cup \set{\bot}$, where the $i$-th output $\tau_i$ is equal to the input token that immediately follows the rightmost previous occurrence of the input token $\sigma_i$, or $\bot$ if there is no such previous occurrence.
    That is:
    $\tau_i = \bot$ if $\sigma_j \neq \sigma_i$ for all $j < i$, and otherwise
    $\tau_i = \sigma_{j_i+1}$ where $j_i = \max\set{ j : j < i \wedge \sigma_j = \sigma_i }$.
\end{itemize}
\citet{sanford2024transformers} call this the "$1$-hop induction heads task"; they also define and study generalizations of the task with increasing difficulty, which they call "$k$-hop" (for $k \in \N$).
The special case of $2$-hop is related to the function composition task defined and studied by \citet{peng2024limitations}.

\citet{bietti2023birth} gave an explicit construction of a transformer for solving the ($1$-hop) induction heads task.
Their construction is a two-layer transformer with a single attention head in each layer.
They also empirically found it difficult to train one-layer transformers to successfully solve the induction heads task under a certain data generation model, but training two-layer transformers was possible.
Indeed, \citet{elhage2021mathematical} noted: ``In small two-layer attention-only transformers, composition seems to be primarily used for one purpose: the creation of [\ldots] induction heads.''

In this note, we prove that a one-layer transformer cannot solve the induction heads task unless the size of the transformer is very large.
By "size", we mean the product of
the number of self-attention heads $h$,
the embedding dimension $m$, and
the number of bits of precision $p$ used by the transformer.
By "very large", we mean that $hmp = \Omega(n)$, where $n$ is the size of the input.
We note that when $\card{\Sigma} \leq n$, there is a two-layer transformer that solves the induction heads task with $h = O(1)$, $m = O(1)$, and $p = O(\log(n))$~\citep{bietti2023birth,sanford2024transformers}.
So our size lower bound for one-layer transformers is exponentially larger than the size that is sufficient for two-layer transformers.

The proof is based on a simple communication complexity argument.
Lower bounds on the size of one-layer transformers that solve related tasks were given by \citet{sanford2023representational} using similar arguments.
Conditional lower bounds for the $k$-hop (for general $k$, mentioned above) were given by \citet{sanford2024transformers} for $\Omega(\log k)$-layer transformers, assuming the 1-vs-2 cycle conjecture in the Massively Parallel Computation model~\citep{im2023massively}.
\citet{peng2024limitations} prove an average-case lower bound for one-layer transformers to solve their function composition task (which resembles the $2$-hop), again using a communication complexity argument.

\section{Transformer model}

In this section, we give a generic definition of one-layer transformers that allows for general forms of token embeddings and positional encodings.
A \emph{self-attention head} with \emph{embedding dimension} $m \in \N$ is a triple $(Q,K,V)$ where $Q \colon \N \times \R^m \to \R^m$, $K \colon \N \times \R^m \to \R^m$, and $V \colon \N \times \R^m \to \R^m$ are, respectively, called the query, key, and value embedding functions, and the first arguments to $Q, K, V$ enable the commonly-used positional encoding.
The self-attention head defines a mapping $\SA_{Q,K,V} \colon \R^{n \times m} \to \R^{n \times m}$ as follows.
On input $X = [ X_1, \dotsc, X_n ]^\T \in \R^{n \times m}$, the output $Y = [ Y_1, \dotsc, Y_n]^\T = \SA_{Q,K,V}(X) \in \R^{n \times m}$ is defined by
\begin{align*}
  Y_i
  & = \sum_{j=1}^n \alpha_{i,j} V(j,X_j) \\
  \quad \text{where} \quad
  (\alpha_{i,1},\dotsc,\alpha_{i,n})
  & = \softmax\del*{ \ip{Q(i,X_i), K(1,X_1)}, \dotsc, \ip{Q(i,X_i), K(n,X_n)} }
  \\
  & = \frac{\del*{ \exp\del{ \ip{Q(i,X_i), K(1,X_1)} }, \dotsc, \exp\del{ \ip{Q(i,X_i), K(n,X_n)} } }}{\sum_{j=1}^n \exp\del*{ \ip{Q(i,X_i), K(j,X_j)} }}
  .
\end{align*}
A \emph{one-layer transformer} with $h$ self-attention heads and embedding dimension $m$ is a collection of $h$ self-attention heads $(Q_t,K_t,V_t)_{t=1}^h$ each with embedding dimension $m$, together with an input encoding function $\encode \colon \N \times \Sigmain \to \R^m$ and an output decoding function $\decode \colon \N \times \R^m \to \Sigmaout$.
Here, the input alphabet $\Sigmain$ and the output alphabet $\Sigmaout$ are finite sets.
The transformer defines a mapping $\TF_{((Q_t,K_t,V_t)_{t=1}^h,\encode,\decode)} \colon \Sigmain^n \to \Sigmaout^n$ as follows.
On input $\sigma = (\sigma_1,\dotsc,\sigma_n) \in \Sigmain^n$, the output $\tau = (\tau_1,\dotsc,\tau_n) = \TF_{((Q_t,K_t,V_t)_{t=1}^h,\encode,\decode)}(\sigma) \in \Sigmaout^n$ is defined by
\begin{equation*}
  \tau_i = \decode\del*{ i, \encode\del*{ i, X_i} + Z_i }
  \quad \text{where} \quad
  [Z_1,\dotsc,Z_n]^\T = \sum_{t=1}^h \SA_{Q_t,K_t,V_t}\del*{ [\encode(1,X_1),\dotsc,\encode(n,X_n)]^\T } .
\end{equation*}
We say that a transformer uses \emph{$p$ bits of precision} if the outputs of all embedding functions ($Q_t$, $K_t$, $V_t$, $\encode$) and the self-attention heads may be rounded to rational numbers with at most $p$ bits of precision without changing the behavior of the mapping $\TF_{((Q_t,K_t,V_t)_{t=1}^h,\encode,\decode)}$.

\section{Size of one-layer transformers for the induction heads task}

\begin{theorem}
  If a one-layer transformer with $h$ self-attention heads, embedding dimension $m$, and $p$ bits of precision solves the induction heads task for input sequences of length $n$ over a three-symbol alphabet, then $hmp = \Omega(n)$.
\end{theorem}
\begin{proof}
  We give a reduction from the one-way communication problem INDEX~\citep[Example 4.19]{kushilevitz1997communication}.
  In this problem, Alice is given a bit string $f = (f_1,\dotsc,f_k) \in \{0,1\}^k$, and Bob is given an index $i^* \in [k]$.
  Alice can send a message to Bob, and after receiving it, Bob has to output $f_{i^*}$.
  By the pigeonhole principle, in any protocol for INDEX, Alice must send at least $k$ bits to Bob.

  Suppose there is a one-layer transformer (with $h$ self-attention heads and embedding dimension $m$, using $p$ bits of precision) that solves the induction heads task with a three-symbol alphabet $\Sigma = \set{0,1,?}$ (so $\Sigmain = \Sigma$ and $\Sigmaout = \Sigma \cup \set{\bot}$).
  We show that it specifies a one-way communication protocol for INDEX.
  Consider the input $n$-tuple $\sigma$, with $n=2k+1$, defined by
  \begin{equation*}
    \sigma = (e_1, f_1, e_2, f_2, \dotsc, e_k, f_k, ?) \in \set{0,1,?}^{2k+1}
  \end{equation*}
  where $f_1,\dotsc,f_k$ are taken from Alice's input, and
  \begin{equation*}
    e_i =
    \begin{cases}
      ? & \text{if $i = i^*$} , \\
      0 & \text{if $i \neq i^*$} ,
    \end{cases}
  \end{equation*}
  which is based on Bob's input.
  The $(2k+1)$-th output of the transformer on input $\sigma$ is $f_{i^*}$, which is exactly the correct output for INDEX.
  We show that Alice can send a message to Bob such that, subsequently, Bob can compute the $(2k+1)$-th output of the transformer and hence determine $f_{i^*}$.

  Consider one of the self-attention heads $(Q,K,V)$, and define $\tilde Q := Q \circ \encode$, $\tilde K = K \circ \encode$, and $\tilde V := V \circ \encode$.
  (We leave out the positional arguments in $Q, K, V, \encode$ for brevity.)
  The $(2k+1)$-th output of $\SA_{Q,K,V}((\encode(\sigma_1),\dotsc,\encode(\sigma_n)))$ is
  \begin{equation*}
    Y_{2k+1} =
    \frac{
      \underbrace{
        \sum_{i=1}^k \exp( \tilde Q(?)^\T \tilde K(e_i) ) \tilde V(e_i)
      }_{\text{known to Bob}}
      +
      \underbrace{
        \sum_{i=1}^k \exp( \tilde Q(?)^\T \tilde K(f_i) ) \tilde V(f_i)
      }_{\text{known to Alice}}
      +
      \underbrace{
        \exp( \tilde Q(?)^\T \tilde K(?) ) \tilde V(?)
      }_{\text{known to both}}
    }{
      \underbrace{
        \sum_{i=1}^k \exp( \tilde Q(?)^\T \tilde K(e_i) )
      }_{\text{known to Bob}}
      +
      \underbrace{
        \sum_{i=1}^k \exp(\tilde Q(?)^\T \tilde K(f_i) )
      }_{\text{known to Alice}}
      +
      \underbrace{
        \exp( \tilde Q(?)^\T \tilde K(?) )
      }_{\text{known to both}}
    }
    .
  \end{equation*}
  In order for Bob to compute $Y_{2k+1}$ to $p$ bits of precision,
  it suffices for Alice to send the values
  \begin{equation*}
    \frac
    {
      \sum_{i=1}^k \exp(\tilde Q(?)^\T \tilde K(f_i)) \tilde V(f_i)
    }
    {
      \sum_{i=1}^k \exp(\tilde Q(?)^\T \tilde K(f_i))
    }
    \in \R^m
    \quad\text{and}\quad
    \log\del*{ \sum_{i=1}^k \exp(\tilde Q(?)^\T \tilde K(f_i)) }
    \in \R
  \end{equation*}
  rounded to $O(p)$ bits of precision
  to Bob in a message of size $Cmp$ bits for some constant $C>0$ (see \Cref{sec:precision} for details).
  Such messages can be sent
  for all $h$ self-attention heads simultaneously; in total, Alice sends just $Chmp$ bits to Bob, and after that, Bob computes the output of the transformer and hence determines $f_{i^*}$, thereby solving the INDEX problem.
  Since every protocol for INDEX must require Alice to send at least $k$ bits, we have
  \begin{equation*}
    hmp \geq \frac{k}{C} = \frac{n-1}{2C} = \Omega(n)
    .
    \qedhere
  \end{equation*}
\end{proof}

\bibliographystyle{plainnat}
\bibliography{1hop}

\appendix

\section{Precision details}
\label{sec:precision}

The $j$-th component of $Y_{2k+1}$ can be expressed as $(A+B) / (Z_A + Z_B)$ where
\begin{align*}
  A
  & = \sum_{i=1}^k \exp( (\tilde Q(?)^\T \tilde K(f_i) ) \tilde V(f_i)_j
  ,
  & B
  & = \exp( \tilde Q(?)^\T \tilde K(?) ) \tilde V(?)_j + \sum_{i=1}^k \exp( \tilde Q(?)^\T \tilde K(e_i) ) \tilde V(e_i)_j
  ,
  \\
  Z_A
  & = \sum_{i=1}^k \exp( \tilde Q(?)^\T \tilde K(f_i) )
  ,
  & Z_B
  & = \exp( \tilde Q(?)^\T \tilde K(?) ) + \sum_{i=1}^k \exp( \tilde Q(?)^\T \tilde K(e_i) ) .
\end{align*}
Define $r := A / Z_A$ and $s := \log Z_A$.
Alice's message to Bob contains $\hat r$ and $\hat s$, which are obtained by rounding $r$ and $s$, respectively, to $3p$ bits of precision.
Hence, $\hat r$ and $\hat s$ satisfy $\abs{r - \hat r} \leq \epsilon$ and $\abs{s - \hat s} \leq \epsilon$ where $\epsilon := 2^{-3p}$.
It suffices to show that Bob can approximate $(A+B)/(Z_A+Z_B)$ up to error $2^{-p}$ using
\begin{equation*}
  \frac{\hat r e^{\hat s} + B}{e^{\hat s} + Z_B}
  ,
\end{equation*}
which only depends on $\hat r$, $\hat s$, $B$, and $Z_B$.
Observe that
\begin{equation*}
  \frac{\hat r e^{\hat s} + B}{e^{\hat s} + Z_B} - \frac{A+B}{Z_A+Z_B}
  = 
  \underbrace{
    r \del*{ \frac{e^{\hat s}}{e^{\hat s} + Z_B} - \frac{e^s}{e^s + Z_B} }
  }_{T_1}
  +
  \underbrace{
    \frac{\del{ \hat r - r } e^{\hat s}}{e^{\hat s} + Z_B}
  }_{T_2}
  +
  \underbrace{
    \frac{B}{Z_A + Z_B} \del*{ \frac{e^s - e^{\hat s}}{e^{\hat s} + Z_B} }
  }_{T_3}
  .
\end{equation*}
By assumption on the embedding functions, we may assume without loss of generality that $\abs{\tilde V(\sigma_i)_j} \leq 2^p$ for all $\sigma_i \in \Sigma$.
Therefore
\begin{equation*}
  \abs{r} \leq \max_{i\in[k]} \abs{\tilde V(f_i)_j} \leq 2^p
  \quad \text{and} \quad
  \frac{\abs{B}}{Z_A+Z_B} \leq \frac{\abs{B}}{Z_B} \leq \max\cbr*{ \abs{\tilde V(?)_j} , \max_{i\in[k]} \abs{\tilde V(e_i)_j} } \leq 2^p .
\end{equation*}
We now bound each of $T_1$, $T_2$, and $T_3$ in magnitude.
To bound $\abs{T_1}$, we use the $(1/4)$-Lipschitzness of the sigmoid function:
\begin{equation*}
  \abs{T_1}
  = \abs{r} \abs*{ \frac{e^{\hat s}}{e^{\hat s} + Z_B} - \frac{e^s}{e^s + Z_B} }
  \leq \frac{\abs{r} \abs{\hat s - s}}4
  \leq \frac{2^p\epsilon}{4} .
\end{equation*}
To bound $\abs{T_2}$:
\begin{equation*}
  \abs{T_2}
  = \abs{\hat r - r} \frac{e^{\hat s}}{e^{\hat s} + Z_B} \leq \abs{\hat r - r} \leq \epsilon .
\end{equation*}
Finally, to bound $\abs{T_3}$, we use the approximation $e^t (e^t - 1) \leq 1.25t$ for $t \in \intcc{0,1/8}$:
\begin{equation*}
  \abs{T_3}
  = \frac{\abs{B}}{Z_A + Z_B} \abs*{ \frac{e^s - e^{\hat s}}{e^{\hat s} + Z_B} }
  \leq \frac{\abs{B}}{Z_A + Z_B} \cdot \frac{Z_A}{Z_A+Z_B} \cdot \frac{e^{\abs{\hat s - s}} - 1}{e^{-\abs{\hat s - s}}}
  \leq 2^p \cdot 1.25\epsilon .
\end{equation*}
Therefore
\begin{equation*}
  \abs*{
    \frac{\hat r e^{\hat s} + B}{e^{\hat s} + Z_B}  - \frac{A+B}{Z_A+Z_B}
  }
  \leq \abs{T_1} + \abs{T_2} + \abs{T_3}
  \leq \del*{ \frac32 \cdot 2^p + 1 } \epsilon
  \leq 2^{-p}
  .
\end{equation*}

\end{document}